\documentclass[10pt,journal,compsoc]{IEEEtran}
\ifCLASSOPTIONcompsoc\usepackage{subcaption}
\usepackage{graphicx}
\usepackage{epstopdf}
\usepackage{caption}
\usepackage{multirow}
\usepackage{adjustbox}
\usepackage{tabularx}
\usepackage{amsmath}
\usepackage{amssymb}
\usepackage{color}
\usepackage{algorithm}
\usepackage{algpseudocode}
\usepackage{scrextend}
\usepackage{todonotes}
\usepackage{lipsum}
\usepackage{amsthm}
\usepackage{enumitem}
\usepackage{multicol}
\usepackage{sidecap}
\usepackage[nocompress]{cite}
\else
 \usepackage{cite}
\fi
\ifCLASSINFOpdf
\else
\fi

\def\bs{\boldsymbol}
\DeclareMathOperator*{\argmin}{arg\,min}

\newtheorem{theorem}{Theorem}[section]

\newtheorem{definition}{Definition}[section]

\begin{document}
\title{Bounded Manifold Completion}

\author{Kelum~Gajamannage and Randy~Paffenroth
\IEEEcompsocitemizethanks{\IEEEcompsocthanksitem K. Gajamannage is with the Department
of Mathematics and Statistics, Texas A\&M University--Corpus Christi, Corpus Christi, TX 78412.\protect\\
E-mail: kelum.gajamannage@tamucc.edu
\IEEEcompsocthanksitem R. Paffenroth is with the Department
of Mathematical Sciences, Department of Computer Science, and Data Science Program, Worcester Polytechnic Institute, Worcester, MA 01609.\protect\\
E-mail:  rcpaffenroth@wpi.edu}

\thanks{}}

\markboth{}
{Shell \MakeLowercase{\textit{et al.}}: Bare Advanced Demo of IEEEtran.cls for IEEE Computer Society Journals}

\IEEEtitleabstractindextext{
\begin{abstract}
Nonlinear dimensionality reduction or, equivalently, the approximation of high-dimensional data using a low-dimensional nonlinear manifold is an active area of research. In this paper, we will present a thematically different approach to detect the existence of a low-dimensional manifold of a given dimension that lies within a set of bounds derived from a given point cloud. A matrix representing the appropriately defined distances on a low-dimensional manifold is low-rank, and our method is based on current techniques for recovering a partially observed matrix from a small set of fully observed entries that can be implemented as a low-rank Matrix Completion (MC) problem. MC methods are currently used to solve challenging real-world problems, such as image inpainting and recommender systems, and we leverage extent efficient optimization techniques that use a nuclear norm convex relaxation as a surrogate for non-convex and discontinuous rank minimization. Our proposed method provides several advantages over current nonlinear dimensionality reduction techniques, with the two most important being theoretical guarantees on the detection of low-dimensional embeddings and robustness to non-uniformity in the sampling of the manifold. We validate the performance of this approach using both a theoretical analysis as well as synthetic and real-world benchmark datasets.
\end{abstract}

\begin{IEEEkeywords}
Manifold, low-rank matrix completion, positive semi-definite, truncated nuclear norm, Gramian.
\end{IEEEkeywords}}

\maketitle

\IEEEdisplaynontitleabstractindextext
\IEEEpeerreviewmaketitle

\ifCLASSOPTIONcompsoc
\IEEEraisesectionheading{\section{Introduction}\label{sec:introduction}}
\else

\section{Introduction}
\label{sec:introduction}
\fi

\IEEEPARstart{L}{ow}-rank Matrix Completion (MC), a technique for estimating unobserved or partially observed (i.e., bounded) entries of a matrix that satisfies a low-rank assumption has been prominent in many real-world applications such as image inpainting \cite{Komodakis2006,Rasmussen2005}, video denoising \cite{Ji2010}, and recommender systems \cite{Koren2008,Steck2010}. Since the data sampled from these applications are naturally high-dimensional and large volume, working with all the available data or acquiring all the information for each matrix entry is often not feasible. Thus, MC provides efficient tools to deal with partially observed or unobserved entries of a matrix and recover a low-rank estimate of the fully observed matrix. Herein, we propose a low-rank MC algorithm with an additional constraint which requires that the recovered matrix represents distances on a manifold. This novel method is theoretically justified based upon recent work in truncated nuclear norm schemes \cite{hu2012fast} and provides superior empirical performance in challenging test problems. 

To set the stage for the description of our approach, we observe that classic techniques, such as Principal Component Analysis (PCA) \cite{Hotelling1933} and Multi-Dimensional Scaling (MDS) \cite{lee2004nonlinear}, assume that the data lay close to a low-dimensional \emph{linear} manifold (or subspace). Advanced techniques, such as Isometric Mapping (ISOMAP) \cite{Tenenbaum2000}, assume that each point in a dataset lay on or near a non-linear manifold where the distances on the manifold provide a faithful \emph{geodesic} representations of the distances between the points in the ambient high-dimensional Euclidean space in which the points are embedded.  Such methods rely on tight theoretical connection between the inner product matrix of the input data, often called the Gramian matrix, and the matrix of \emph{squared} Euclidean Distances (EDs) between the points in the data-set \cite{lee2004nonlinear}.  One recent method that is somewhat different than the above, but superficially similar in flavor to our proposed technique is Maximum Variance Unfolding (MVU) \cite{weinberger2006unsupervised}, which attempts to maximize the variance of the data in the presence of distance constraints between some pairs of points.  However, as we will explain in detail in the sequel, our proposed method arises from a quite different theoretical perspective than MVU and extends the capabilities of current methods in several important directions.

With the above ideas in mind, the high-level idea of our proposed algorithm is actually quite simple.  Given a collection of points in high-dimensional space, a classic MDS algorithm assumes that the EDs between all the points are a faithful representation of "meaningful" distances between the points (an idea that will be made precise in the sequel).  ISOMAP relaxes that assumption by only using EDs between neighboring points and using a shortest path \emph{graph distance} as a proxy for a "meaningful" distance between far away points.  Our method relaxes this assumption even further.  Like ISOMAP, we assume that the EDs between neighboring points are "meaningful", but \emph{differing from ISOMAP, we do not assume we have any approximation for far away distances}. We only assume that there exist \emph{rough upper and lower bounds for the distances between far away points} (perhaps provided by a maximum curvature type argument).  In fact, there can be pairs of points (perhaps even many of them) for which nothing is known about their distances on the low-dimensional manifold.  Such upper and lower bounds, combined with a constraint that the recovered matrix is Positive Semi-Definite (PSD), allow us to apply a MC scheme for the recovery of the unobserved distances and thereby the parameterization of the underlying low-dimensional manifold.  Therefore, we name our method Bounded Manifold Completion (BMC). 

Recovery of an arbitrary matrix using a low-rank principle is an ill-posed problem since such an optimization problem is non-convex and discontinuous \cite{hu2012fast}. Recent research in MC \cite{wright2009robust,candes2011robust, paffenroth2013space}, suggests that the nuclear norm convex relaxation is an efficient alternative to rank. However, the major limitation of an optimization involving a nuclear norm is that it minimizes the sum of \emph{all the Singular Values (SVs)}, and therefore might not correctly approximate the distances on the manifold. Thus, similar to \cite{hu2012fast}, here we employ a \emph{truncated nuclear norm} (equivalently, the sum of the smallest SVs) to achieve an accurate approximation for the partially observed distance matrix. We formulate BMC as an optimization scheme with multiple constraints where we use an Alternating Direction Method of Multipliers (ADMM) \cite{boyd2010distributed}, to reformulate the original problem as a general separable collection of convex problems \cite{hu2012fast} with an additional constraint to enforce that the returned matrix is a distance matrix. Existing schemes utilizing nuclear norm heuristics such as SV thresholding \cite{Cai2010}, nuclear norm regularized least squares \cite{Toh2009}, and Robust Principal Component Analysis (RPCA)  \cite{wright2009robust} have shown noteworthy performance.   However, as we show in Fig.~\ref{fig:MCintro}, using a truncated nuclear norm for non-linear dimensionality reduction allows the detection of a low-dimensional embedding with \emph{an error that is $0$, to within machine-precision.} 

\begin{figure}[htp]
    	\centering  
        	\includegraphics[width=3.5in]{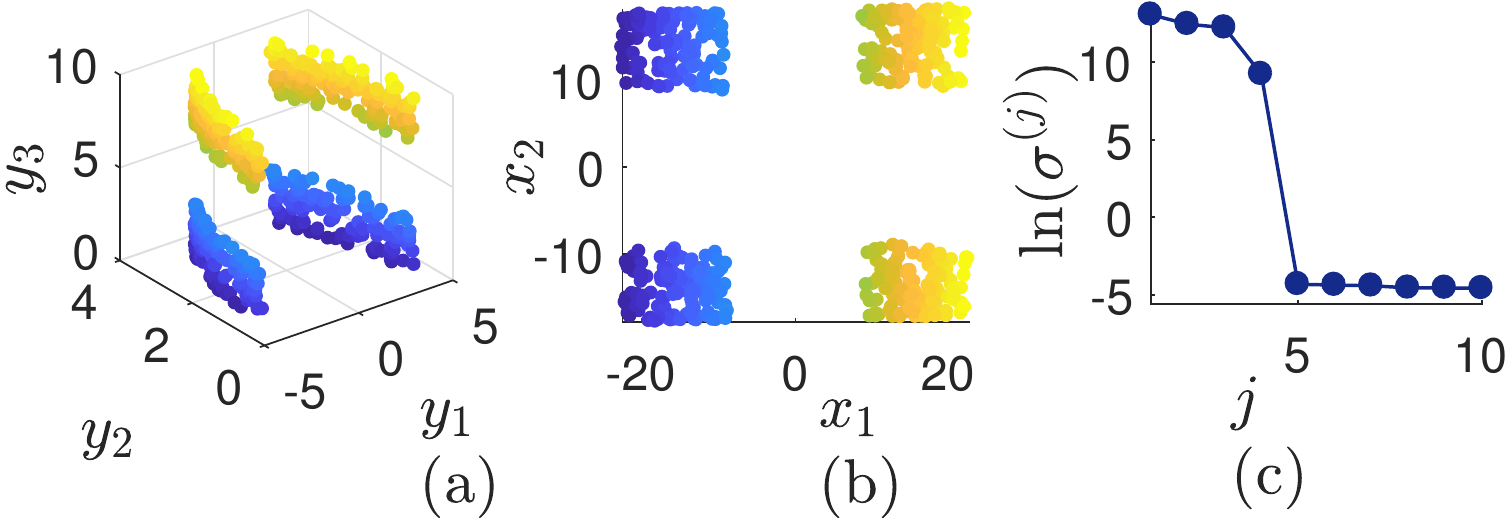}
        	\caption{Recovery of a non-uniformly sampled non-linear manifold using BMC. (a) A dataset of 500 points sampled \emph{non-uniformly} from a semi-cylinder having a hollow region and four isolated clusters. Here, the color spectrum shows the variation of manifold lengths along the $y^{(3)}$ axis. (b) 2-D embedding of the dataset performed using BMC where the color variation is associated with that of (a).  Note, the correct 2-D embedding is accurately recovered even though the non-linear manifold is sampled non-uniformly. (c) Semilogarithmic plot of the first 10 SVs ($\sigma^{(j)}$'s) of the recovered squared distance matrix.  For this data set, the theoretically optimal embedding, by construction, has 4 non-zero singular values.  Note that the fifth singular value is smaller than the fourth singular value by more than \emph{10 orders of magnitude} and is merely an artifact of machine precision.}
	\label{fig:MCintro}
\end{figure}

One of the key elements of our proposed MC scheme is the MC algorithm from \cite{paffenroth2013space} which is capable of recovering a matrix also having partially observed data. Specifically, the input distance matrix here is given in terms of two matrices, an upper bound matrix and a lower bound matrix.  In particular, the two corresponding values in the bound matrices are the same when the value at that location is fully observed and the two values in the bound matrices are not equal when the value of that location is partially observed. In particular, unobserved data can also be handled in this way by setting the lower and upper bounds to be arbitrarily small and large, respectively. This partially observed data technique is implemented in BMC which allows our method to leverage distances that are fully observed (the Manifold Distances (MDs) between two points are known), partially observed (the MDs between two points can only be bounded), and unobserved (the MDs between the points are not known). 

Referring back to Fig.~\ref{fig:MCintro}, we provide a simple example to illustrate the performance of BMC. We sample a dataset of 500 points $\bs{Y}=\{\bs{y}_1,\cdots,\bs{y}_{500} \}$ with $\bs{y}_i = \left(y^{(1)}_i, y^{(2)}_i, y^{(3)}_i\right) \in \mathbb{R}^{3}$ from a semi-cylinder such that the sample has some hollow parts of unknown size and four isolated clusters in the corners of the data [Fig.~\ref{fig:MCintro}(a)]. We construct a \emph{squared} ED matrix $D^e\in\mathbb{R}^{500 \times 500}_{\ge0}$ (real numbers greater than or equal to $0$) such that $D^e_{ij} = \|\bs{y}_i -\bs{y}_j\|^2_2$. We set the lower bound matrix ($D^l$) and the upper bound matrix ($D^u$) of the recovered \emph{squared} distance matrix to be $D^l_{ij}= 0.1 D^e_{ij}$ and $D^u_{ij}= 10 D^e_{ij}$.  \emph{Note, no actual MDs are required.} A semi-cylinder represents a 2-D manifold; thus, the rank of the squared distance matrix of this dataset, when properly embedded, is four (we present this as a Theorem in Sec.~\ref{sec:mc}). Then, BMC truncates the first four SVs, generates the recovered distance matrix $L$, and performs a 2-D embedding of the dataset in Fig.~\ref{fig:MCintro}(b) . We observe in the embedding that BMC is capable of projecting the data onto a 2-D plane such that the topology of the data is preserved even though the data is non-uniformly sampled. We compute the natural logarithm of the first 10 SVs of BMC's recovered squared distance matrix in Fig.~\ref{fig:MCintro}(d). We observe that the first four SVs of $L$ are non-zero while the rest of the SVs are (to within machine precision) zero and, therefore, BMC is capable of recovering the correct low-rank representation of the data.

All the notations and abbreviations used in the paper, along with their descriptions, appear in Table~\ref{tab:nomenclature}. This paper is structured as follows: In Sec.~\ref{sec:mwthods}, we will detail our proposed method in three subsections; first, we present the derivation of our BMC method along with some classic MC methods that inspired BMC; then, the numerical implementation of BMC is presented; finally, we present the ADMM solver of BMC. We validate the performance of BMC in Sec.~\ref{sec:per_analysis} using one synthetic dataset, a cylinder having a hollow region, and two real-life datasets, face images and handwritten digits. Here, we also present five classic DR methods that we will contrast/compare with BMC. Finally, we present discussions and conclusions in Sec.~\ref{sec:conclusion}. 

\subsection{Contributions}
Our proposed BMC approach makes the following contributions to the literature:
\begin{itemize}
\item BMC is a novel non-linear dimension reduction algorithm which leverages upper and lower bounds on MDs rather than parameters such as the presumed number of neighbors of a data point.
\item Even more, BMC is capable of operating when no information is available for many of the MDs.
\item BMC provides theoretical guarantees as to the performance of the algorithm, leveraging ideas from \cite{hu2012fast}.
\item BMC provides unique capabilities for detecting low-dimensional manifolds even when the manifold is non-uniformly sampled.
\item BMC is implemented using a fast ADMM solver allowing large problems to be treated with many data points.
\item The idea of using bounded low-rank matrix completion has applications outside of low-dimensional manifold detection where the input data is noisy or otherwise only roughly known. 
\end{itemize}

\begin{table*}[htp]
\caption {Nomenclature} \label{tab:nomenclature}
\begin{multicols}{2}
  
\begin{tabular}{p{3cm}|p{6.2cm}}
Notation &  Description \\
\hline\\
$\mathcal{S}$ &  Similarity metric \\
$\mathcal{M}_{ij}$ &  Manifold distance between points $i$ and $j$ \\
$\sigma^{(j)}$ & $j$-th singular value \\
$\lambda^{(j)}$ & $j$-th eigenvalue \\
$\lambda^{(j)}_m$ & $j$-th eigenvalue at the $m$-th iteration\\
$n$ & Number of points in the dataset \\
$U$ and $V$ & Unitary matrices \\
$D^e$ & Squared Euclidean distance matrix \\
$D^l$ & Lowe bound distance matrix \\
$D^u$ & Upper bound distance matrix \\
$\alpha^l$ & Scaling matrix associated with $D^l$ \\
$\alpha^u$ & Scaling matrix associated with $D^u$ \\
$L, K$ & Recovered distance matrices \\
$\mathrm{diag}(L)$ &  Vector formed from the diagonal of the matrix $L$ \\
$\mathrm{Diag}(\bs{v})$ &  Diagonal matrix formed with $\bs{v}$ as its diagonal \\
$\mathrm{Gram} (L)$ & Gramian of the matrix $L$ \\
$\mathrm{tr}(A)$ &  Trace of the matrix $A$ \\
$A\succeq 0$ & A is positive semi-definite \\
$E_{ij}=E(D^l_{ij}, D^u_{ij}, L_{ij})$ & Shrinkage operator of $L$ for bounds $D^l$ and $D^u$ \\
$\bs{y}_i=\left(y^{(1)}_i,\dots,y^{(d)}_i\right)^T$ & $i$-th point in the dataset of $d$-dimensions \\
$\bs{x}_i=\left(x^{(1)}_i,\dots,x^{(e)}_i\right)^T$ & $i$-th point in the $e$-dimensional embedding \\
$\zeta, \eta$ & Lagrangian multipliers \\ 
$\delta$ & Number of nearest neighbors \\
$m$ & index for the iteration \\
$T$ & total number of iterations \\
$\rho^{\zeta}_m, \rho^{\eta}_m$ & Lagrangian vectors at $m$-th iteration \\
$\|A\|_F$ & Frobenius norm of $A$ \\
$\langle A, B\rangle$ & Frobenius inner product of $A$ and $B$ \\
\end{tabular}

\begin{tabular}{p{.5cm} p{1.5cm}|p{5.5cm}}
& Notation & Description\\
\cline{2-3}
& & \\
& $\mathcal{D}_{\nu}$ & SV shrinkage operator \\
& $\bs{e}$ & Column vector of ones \\
& $r$ & Number largest SV to be truncated \\
& $I$ & Identity matrix \\
& $\mathcal{E}_c$ & Clustering error \\
& $\mathcal{E}_n$ & Neighborhood preserving error\\
\multicolumn{3}{c}{} \\

& Abbreviation & Description\\
\cline{2-3}
& & \\
& ADMM&Alternating direct method of multipliers\\
& ALF&Augmented Lagrangian function\\
& ALM&Augmented Lagrangian multipliers\\
& BMC& Positive semi-definite matrix completion\\
& DM&Diffusion maps\\
& DR&Dimensionality reduction\\
& ED&Euclidean distance\\
& EV&Eigenvalue\\
& EVD&Eigenvalue decomposition\\
& HLLE& Hessian local linear embedding\\ 
& KMC&k-means clustering\\
& LE& Laplacian eigenmaps\\
& LLE& Local linear embedding\\
& MC&Matrix completion\\
& MD&Manifold distance\\
& MDS&Multi-dimensional scaling\\
& NP&Non polynomial\\
& PSD&Positive semi-definite\\
& RPCA&Robust principal component analysis\\
& SV&Singular value\\
& SVD&Singular value decomposition\\
\end{tabular}
\end{multicols}
\end {table*}

\section{Methods}\label{sec:mwthods}
In this section, first, we present the details of the BMC scheme along with the details of extent MC algorithms that inspired the BMC scheme. Then, we present the numerical implementation of BMC. Finally, we derive an iterative ADMM solver for BMC. 

\subsection{Matrix Completion}\label{sec:mc}

In general, low-rank MC schemes take a feature matrix with unobserved entries and recovers the unobserved entries of that matrix by utilizing optimization techniques.  Using such MC techniques, we are interested in a low-rank recovery of distances on a manifold. 

We start the derivation of BMC by assuming that we are given an arbitrary squared distance matrix of size $n\times n$, $L\in\mathbb{R}^{n\times n}_{\ge0}$. For such a matrix $L$ we denote by
\begin{equation}\label{eqn:double_centering}
\mathrm{Gram}(L)_{ij}=-\frac{1}{2}\big[L_{ij}-\mu_i(L_{ij}) -\mu_j(L_{ij})+\mu_{ij}(L_{ij})\big],
\end{equation} 
the Gramian of $L$, where $\mu_i(L_{ij})$, $\mu_j(L_{ij})$, and $\mu_{ij}(L_{ij})$ are the means of the $i$-th row, $j$-th column, and the full matrix $L$, respectively, \cite{lee2004nonlinear}. Ref.~\cite{lee2004nonlinear} states that $L$ is a squared distance matrix if $\mathrm{Gram}(L)$ is PSD, as in Definition~\ref{def:psd}, leading to our use of the Gramian. We will impose this essential condition of $\mathrm{Gram}(L)$ is PSD, denoted by $Gram(L)\succeq 0$, as a constraint to BMC in addition to one general constraint in it, that we will discuss in the sequel.
\begin{definition}\label{def:evd}
Consider that $Diag(\lambda^{(1)}, \dots \lambda^{(n)})$ represents a diagonal matrix formed with the vector $(\lambda^{(1)}, \dots \lambda^{(n)})$ as its diagonal. Let $L\in\mathbb{R}^{n\times n}$ be a diagonalizable matrix (e.g. a distance matrix) then the Eigenvalue Decomposition (EVD) of $L$ is $L=U\Lambda U^{-1}$, where $\Lambda= $$Diag(\lambda^{(1)}, \dots \lambda^{(n)})$. Here, for $j=1 ,\dots, n$, $\lambda^{(j)}$ represents $j$-th eigenvalue of $L$.
\end{definition}
\begin{definition}\label{def:psd}
Let $\lambda^{(j)}$ is the $j$-th Eigenvalue (EV) of $\mathrm{Gram}(L)$, then the matrix $\mathrm{Gram}(L)\in\mathbb{R}^{n\times n}$ is PSD if  $\lambda^{(j)}\ge0$ for all $j=1, \dots,  n$. We denote PSD $\mathrm{Gram}(L)$ by $\mathrm{Gram}(L)\succeq 0$.
\end{definition}

As seen in Fig.~\ref{fig:MCintro}, in contrast to many other DR methods, both, 1) BMC is capable of recovering a single manifold even when the sampling of the manifold is non-uniform, 2) BMC allows many of these distances to be only \emph{partially observed}. BMC accomplishes this task by using Euclidean distances between the given points as similar to many other DR methods. Consider that the MD between two points $\bs{y}_i$ and $\bs{y}_j$ is denoted by $\mathcal{M}_{ij}$.  We then define the squared lower bound matrix, denoted by $D^l=[D^l_{ij}]_{n\times n}\in \mathbb{R}_{\ge0}^{n\times n}$, and the squared upper bound matrix, denoted by $D^u=[D^u_{ij}]_{n\times n}\in \mathbb{R}_{\ge0}^{n\times n}$,  as
\begin{equation}\label{eqn:lowerBound}
	D^l_{ij} = 
	\left\{
	\begin{array}{ll}
		\mathcal{M}^2_{ij}, & \text{if the MD between}\\
		 & \text{points} \ \bs{y}_i \ \text{and} \ \bs{y}_j \text{are observed,} \\
			  \alpha^l_{ij} \|\bs{y}_i-\bs{y}_j\|^2, & \mathrm{otherwise},      	
	\end{array} 
	\right.
\end{equation}
and
\begin{equation}\label{eqn:upperBound}
	D^u_{ij} = 
	\left\{
	\begin{array}{ll}
		\mathcal{M}^2_{ij}, &\text{if the MD between}\\
		 & \text{points} \ \bs{y}_i \ \text{and} \ \bs{y}_j \text{are observed,} \\
			  \alpha^u_{ij} \|\bs{y}_i-\bs{y}_j\|^2, & \mathrm{otherwise},
	\end{array} 
	\right.
\end{equation}
respectively. We say that $\mathcal{M}_{ij}$ is fully observed if $\mathcal{M}^2_{ij}=D^l_{ij}=D^u_{ij}$.  Otherwise, we set suitable values for $\alpha^l=[\alpha^l_{ij}]_{n\times n}$ and $\alpha^u=[\alpha^u_{ij}]_{n\times n}$, where $\alpha^l_{ij}\le\alpha^u_{ij}$ for all $ij$, which are user defined non-negative scaling matrices of the squared lower bound matrix and  the squared upper bound matrix, respectively. These parameters allow some MDs to only be bounded rather than known precisely. How might the values for $\alpha^u$ and $\alpha^l$ be set?  Formally speaking, the BMC algorithm admits any choice the user might make.  However, in practice, and as we will demonstrate in the sequel, choosing $\alpha^u$ and $\alpha^l$ based upon a bound on the \emph{curvature} of the desired manifold leads to advantageous properties.

In particular, the scaling matrices $\alpha^l$ and $\alpha^u$ can be used to define all the three scenarios of MDs of interest to the BMC algorithm, namely \emph{observed}, \emph{partially observed}, and \emph{unobserved}. Leveraging an observed distance is straight forward according to \eqref{eqn:lowerBound} and \eqref{eqn:upperBound}. If the MD between the points $\bs{y}_i$ and $\bs{y}_j$ is partially observed, we approximate finite values for $\alpha^l_{ij}$ and $\alpha^u_{ij}$ in \eqref{eqn:lowerBound} and \eqref{eqn:upperBound} such that those values closely bound the squared MD between the points $\bs{y}_i$ and $\bs{y}_j$. If the MD between the points $\bs{y}_i$ and $\bs{y}_j$ is unobserved, approximating the values for the parameters $\alpha^l_{ij}$ and $\alpha^u_{ij}$ are uncertain. Thus, we set $\alpha^l_{ij}=0$ and $\alpha^u_{ij}=\infty$ (or some arbitrarily large number).

BMC's objective function minimizes the truncated rank of the recovered matrix $L$, denoted by $Rank_r (L)$, defined as the number of nonzero values of the ($r+1$)-th tail of SVs $\{\sigma^{(r+1)},\dots,\sigma^{(n)}\}$ of $L$ where $r\ge0$ as in Definition~\ref{def:svd}.  I.e., $Rank_r (L)$ counts the number of non-zero SVs \emph{not counting the $r$ largest SVs}.  We are interested in minimizing the truncated rank, as suggested in \cite{hu2012fast}, rather than the rank since it allows the user to omit large SVs from the minimization process. Note, setting $r=0$ is the same as considering the rank of $L$ without a truncation. Thus, the optimization scheme of BMC that we use to recover a matrix $L$ by completing the incomplete matrix $\mathcal{M}^2$ is 
\begin{equation}\label{eqn:optRouting}
\begin{aligned}
& \underset{L}{\text{minimize}} & Rank_r (L), \\
& & D^l\le L \le D^u, \\
&& \mathrm{Gram} (L) \ \succeq 0.\\
\end{aligned}
\end{equation}

\begin{definition}\label{def:svd}
	Consider that $Diag(\sigma^{(1)}, \dots \sigma^{(\min(n,m))})$ represents a diagonal matrix formed with the vector $(\sigma^{(1)}, \dots, \sigma^{(\min(n,m))})$ as its diagonal. Let $L\in\mathbb{R}^{n\times m}$ is a matrix, and $U_{m\times m}$ and $V_{n\times n}$ are unitary matrices such that $U^TU=I$ and $V^TV=I$, respectively. Then, Singular Value Decomposition (SVD) of $L$ is $L=U\Sigma V^T$, where $\Sigma_{m\times n}= $$Diag(\sigma^{(1)}, \dots \sigma^{(\min(n,m))})$. Here, for $j=1 ,\dots, n$, $\sigma^{(j)}$ represents $j$-th singular value of $L$ and $\sigma^{(j)}\ge \sigma^{(j+1)}; \ \forall j$. 
\end{definition}

The optimization scheme in ~\eqref{eqn:optRouting} is inspired by classic theorems relating the rank of distance matrices to dimensions of manifolds.  In particular, for an $e$-D manifold, Theorem~\ref{thm:rank} states that the rank of the squared distance matrix $L$ is $e+1$ or $e+2$ depending on whether the dataset represents a hyper-sphere or a non-hyper-sphere, respectively. Thus, we set $r=e+2$ in \eqref{eqn:optRouting} for a dataset representing non-hyper-spherical $e$-D manifold which is the most common case of the datasets. 

\begin{theorem}\label{thm:rank}
~\cite{gower1985applications}, 
Let the dimensionality of a given manifold is $e$. Then, 
\begin{enumerate}[label=(\roman*)]
\item the rank of the squared distance matrix computed on the manifold is $e+1$ iff the manifold is a hyper-sphere, or
\item the rank of the squared distance matrix computed on the manifold is $e+2$ iff the manifold is not a hyper-sphere.
\end{enumerate}
\end{theorem}
\begin{proof}
See Theorem~6 and the associated derivation in ~\cite{gower1985applications}.
\end{proof}

Computing $Rank_r(L)$ in \eqref{eqn:optRouting} is non-convex and NP-hard\footnote{A problem is NP-hard if an algorithm for solving it can be translated into an algorithm for solving any non-deterministic polynomial time problem.}.  However, as stated in \cite{hu2012fast}, the \emph{truncated nuclear norm}, see Definition.~\ref{def:trun}, provides a proxy for $Rank_r(L)$ that can be optimized efficiently. Thus, inspired by the approach in \cite{hu2012fast}, we approximate $Rank_r(L)$ in  \eqref{eqn:optRouting} by the $r$-th truncated nuclear norm of $L$, denoted by $\|L\|_{*,r}$. With the aforementioned modification, the optimization scheme of BMC is written as 
\begin{equation}\label{eqn:trunNormApprx}
\begin{aligned}
& \underset{L}{\text{minimize}} & \|L\|_{*,r}, \\
& & D^l\le L \le D^u, \\
&& \mathrm{Gram} (L) \ \succeq 0.\\
\end{aligned}
\end{equation}

\begin{definition}\label{def:nucl}
	For a given matrix $L\in\mathbb{R}^{n\times n}$, the nuclear norm of $L$, denoted by $\|L\|_*$, is defined as
	\begin{equation}\label{eqn:nuclrNorm}
	\|L\|_*=\sum^n_{j=1}\sigma^{(j)},
	\end{equation} 
	where $\sigma^{(j)}$ denotes $j$-th SV of $L$.
\end{definition}
	
\begin{definition}\label{def:trun}
	For a given matrix $L\in\mathbb{R}^{n\times n}$ and $r\in \mathbb{Z}_{\ge0}$, the $r$-th truncated nuclear norm of $L$, denoted by $\|L\|_{*,r}$, is defined as
	\begin{equation}\label{eqn:trunNorm}
	\|L\|_{*,r}=\sum^n_{j=r+1}\sigma^{(j)},
	\end{equation} 
	where $\sigma^{(j)}$ denotes the $j$-th SV of $L$.
\end{definition}
	
\subsection{Implementation}

Computation of $\|L\|_{*,r}$ in \eqref{eqn:trunNormApprx} is numerically implemented by  subtracting the sum of $r$-largest SVs from the sum of all the SVs ($\|L\|_*$ is the sum of all the SVs) as stated by Theorem~\ref{thm:trun}. Thus, with the aid of Theorem~\ref{thm:trun}, the modified optimization scheme is
\begin{equation}\label{eqn:convRelx}
\begin{aligned}
& \underset{L}{\text{minimize}} & \|L\|_* - \mathrm{tr}(U_rLV^T_r), \\
& & D^l\le L \le D^u, \\
&& \mathrm{Gram} (L) \ \succeq 0.\\
\end{aligned}
\end{equation}

\begin{theorem}\label{thm:trun} ~\cite{hu2012fast}.
	For a given matrix $L\in\mathbb{R}^{n\times n}$; let, $L=U\Sigma V^T$ provides SVD of $L$. Let, $U_r$ consists of the first $r$ columns of the matrix $U$ and $V^T_r$ consists of the first $r$ columns of $V^T$. Moreover, $tr(U_rLV^T_r)$ denotes the trace of $U_rLV^T_r$,  that is, sum of all the diagonal entries of $U_rLV^T_r$. Then, 
	\begin{equation}\label{eqn:huMC}
		\|L\|_{*,r} = \|L\|_* - \mathrm{tr}(U_rLV^T_r).
	\end{equation}
	Here, $\|L\|_*$ is the nuclear norm of $L$ as given by Definition~\ref{def:nucl}.
\end{theorem}
	
\begin{proof}
See Theorem~3.1 and the following derivation in ~\cite{hu2012fast}.
\end{proof}

To optimize \eqref{eqn:convRelx} we utilize an ADMM \cite{boyd2010distributed}.  In particular, as inspired by \cite{hu2012fast} we split the original decision variable $L$ in the objective function of \eqref{eqn:convRelx} into two decision variables, say $L$ and $K$, such that $L=K$. Without loss of generality, we replace $L$ by $K$ in the second term of the objective function and rewrite the optimization scheme in \eqref{eqn:convRelx} as 
\begin{equation}\label{eqn:optRouting1}
\begin{aligned}
& \underset{L,K}{\text{minimize}} & \|L\|_* - \mathrm{tr}(U_rKV^T_r) \\
& & L=K,\\
& & D^l\le K \le D^u, \\
&& \mathrm{Gram}(L) \ \succeq 0.\\
\end{aligned}
\end{equation}
This transformation allows a closed form solution in the case where we fix $L$ and optimize \eqref{eqn:optRouting1} with respect to $K$, or in the case where we fix $K$ and optimize \eqref{eqn:optRouting1} with respect to $L$, and leads to an efficient ADMM solver.

BMC's technique of leveraging the inequality constraint $D^l\le K \le D^u$ in \eqref{eqn:optRouting1} is influenced by \cite{paffenroth2013space}, where a new function is introduced to imply this inequality. In BMC, we denote this new function as $E=[E_{ij}]_{n\times n}\in \mathbb{R}^{n\times n}$ and model it such that if a recovered entry $K_{ij}$ satisfies the inequality constraint, then $E_{ij}=0$, otherwise  $E_{ij}$ quantifies the violation of the inequality constraint. Thus, 
\begin{equation}
E_{ij}(D^l_{ij}, D^u_{ij}, K_{ij})=
\left\{
\begin{array}{ll}
	K_{ij}-D^l_{ij}, & \mathrm{if} \ D^l_{ij}>K_{ij}, \\
      	K_{ij}-D^u_{ij}, & \mathrm{if} \ D^u_{ij}<K_{ij}, \\     
      	0, & \mathrm{otherwise}, \\
\end{array} 
\right.\\
\end{equation}
which is shown in Fig.~\ref{fig:funE} and is often known as a shrinkage operator in the literature. Note that, for given $D^l_{ij}$ and $D^u_{ij}$, $E_{ij}$ is a function of $K_{ij}$. Finally, to satisfy the inequality constraint in \eqref{eqn:optRouting1}, we add the constraint $E=[0]_{n\times n}$ into the optimization scheme. With the aforesaid modifications, our optimization scheme is
\begin{equation}\label{eqn:optRouting2}
\begin{aligned}
& &\underset{L, K}{\text{minimize}} \ \ \ \|L\|_* - \mathrm{tr}(U_rKV^T_r) \\
& & L=K,\\
& & \mathrm{Gram}(L) \ \succeq 0,\\
& & E=[0]_{n\times n}, \\
& & \text{where} \ \  E_{ij}=
\left\{
\begin{array}{ll}
	K_{ij}-D^l_{ij}, & \mathrm{if} \ D^l_{ij}>K_{ij}, \\
      	K_{ij}-D^u_{ij}, & \mathrm{if} \ D^u_{ij}<K_{ij}, \\     
      	0, & \mathrm{otherwise}. \\
\end{array} 
\right.\\
\end{aligned}
\end{equation}

\begin{figure}[htp]
	\begin{center}
	\includegraphics[width=2.5in]{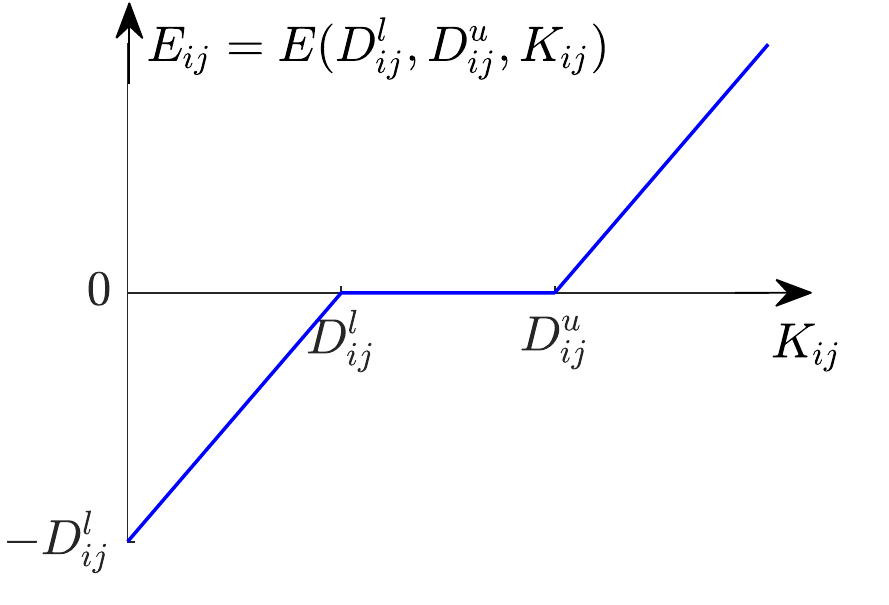}
	\end{center}
	\caption{Function $E_{ij}$ versus $K_{ij}$ for arbitrary upper bound, $D^l_{ij}$, and lower bound, $D^u_{ij}$. $E_{ij}$ is zero if $K_{ij}$ is between the bounds, but it increases with the same rate as $K_{ij}$ if it is outside of the bounds.}
	\label{fig:funE}
\end{figure}

In Sec.~\ref{sec:admm}, we solve the optimization scheme in \eqref{eqn:optRouting2} using an ADMM.

\subsection{ADMM solver}\label{sec:admm}
An ADMM classically uses Augmented Lagrangian Function (ALF) to transform a constrained optimization problem into a related unconstrained optimization problem. Thus, inspired by \cite{boyd2010distributed}, first we formulate the ALF of the optimization scheme in \eqref{eqn:optRouting2} and then we follow an ADMM scheme to minimize the ALF. Let the matrices $\zeta\in \mathbb{R}^{n\times n}$ and $\eta\in \mathbb{R}^{n\times n}$ denote Lagrangian multipliers, $\|\cdot\|_F$ denotes the Frobenius norm as in Definition~\ref{def:frob}, and $\langle \cdot,\cdot\rangle$ denotes the inner product as in Definition~\ref{def:inner_product}.  The ALF of the optimization scheme in \eqref{eqn:optRouting2} is formed by combining all the constraints into the Lagrangian $\mathcal{L}$ as
 \begin{equation}\label{eqn:ALM}
\begin{aligned}
&\mathcal{L}(L,K,\zeta,\eta,\rho^{\zeta},\rho^{\eta}\vert \mathrm{Gram}(L) \succeq 0) \\
&= \|L\|_* - \mathrm{tr}(U_r K V^T_r) + \langle \zeta,L-K \rangle+\frac{\rho^{\zeta}}{2}\|L-K\|^2_F \\
&\hspace{5cm}+ \langle\eta,E\rangle+\frac{\rho^{\eta}}{2}\|E\|^2_F,\\
\end{aligned}
\end{equation}
where the scalars $\rho^{\zeta}>0$ and $\rho^{\eta}>0$ are penalty parameters. Thus,  the minimization of $\mathcal{L}$ is equivalent to the optimization of \eqref{eqn:optRouting2}.

\begin{definition}\label{def:frob}
	The Frobenius norm of a matrix $A=[a_{ij}]_{n\times n}\in\mathbb{R}^{n\times n}$ is defined as
	\begin{equation}
	\|A\|_F=\sqrt{\sum_{i,j=1}^{n}|a_{ij}|^2}.
	\end{equation} 
\end{definition}

\begin{definition}\label{def:inner_product}
For given matrices $A=[a_{ij}]_{n\times n} \in\mathbb{R}^{n\times n}$ and $B=[b_{ij}]_{n\times n} \in\mathbb{R}^{n\times n}$, the Frobenius inner product of them is denoted by $\langle A,B\rangle$ and is defined as
\begin{equation}\label{eqn:nuclrNorm}
\langle A,B\rangle=\sum_{1\le i,j \le n}a_{ij}b_{ij}.
\end{equation} 
\end{definition}
	
Small values for the penalty parameters in minimization of $\mathcal{L}$ emphasize the minimization of the objective function while large values for the parameters emphasize the feasibility (satisfying constraints).  Our ADMM allows a closed-form approximation to minimization of $\mathcal{L}$ by first minimizing $\mathcal{L}$ with respect to one decision variable or one Lagrangian multiplier while keeping all the other variables and multipliers fixed. Then, ADMM uses the current solution to partially update the other variables and multipliers before their minimization and ADMM generates an iterative set of solutions for updating both the decision variables of the optimum objective function and the Lagrangian multipliers that enforce constraints. 

We begin by minimizing $\mathcal{L}$ in \eqref{eqn:ALM} with respect to the variables $L$, $K$, $\zeta$, and $\eta$ in Sec.~\ref{sec:min_ALF_with_L}. 

\subsubsection{Minimization of ALF with respect to $L$}\label{sec:min_ALF_with_L}
Let $T$ denotes the total number of iterations.  We formulate the minimization of the Lagrangian $\mathcal{L}$ for a general iteration, say $m$, where $1\le m\le T$. The solution  to the minimization of $\mathcal{L}$ in \eqref{eqn:ALM} with respect to $L$ at the $m$-th iteration is considered to be $L_{m+1}$. That is,
\begin{equation}
\begin{aligned}\label{eqn:solForK_old}
L_{m+1} &= \underset{L}{\argmin} \ \mathcal{L}(L,K_m,\zeta_m,\eta_m,\rho^{\zeta}_m,\rho^{\eta}_m\vert \mathrm{Gram}(L)\succeq 0)\\
& = \underset{L\vert \mathrm{Gram}(L)\succeq 0}{\argmin}
\left(
\begin{array}{c}
	\|L\|_*+\langle \zeta_m,L-K_m \rangle\\
	+\frac{\rho^{\zeta}_m}{2}\|L-K_m\|^2_F\\
\end{array} 
\right),\\
\end{aligned}
\end{equation}
where $1\le m\le T$. Note that, we ignored the terms independent of $L$ in \eqref{eqn:solForK_old}. We use both, 1) the relationship $\|A\|^2_F=\langle A, A \rangle$, where $A$ is an arbitrary matrix, between the Frobenius norm and the inner product, and 2) some identities of the inner product, to simplify \eqref{eqn:solForK_old} as\begin{equation} \label{eqn:solForK}
L_{m+1} = \underset{L\vert \mathrm{Gram}(L)\succeq 0}{\argmin}\left[\|L\|_*+\frac{\rho^{\zeta}_m}{2}\left\|L-\left(K_m-\frac{1}{\rho^{\zeta}_m}\zeta_m\right)\right\|^2_F\right].
\end{equation}

Now, we provide the Definition~\ref{def:shrink}  \cite{hu2012fast, Cai2010} of a shrinkage operator that shrinks SVs of a matrix $G$ by the value of $\nu \ge 0$ as we are going to use this operator in Theorem~\ref{trm:shrink}. In general, $\nu$ is a small value comparative to the largest SVs. Such a shrinkage operator eliminates small SVs representing non-prominent features of the dataset and retains only the prominent features. 
\begin{definition}\label{def:shrink}\cite{hu2012fast, Cai2010},
Let $G\in\mathbb{R}^{n\times n}$ and SVD of $G$ is $G=U\Sigma V^T$ where $\Sigma=Diag(\{\sigma^{(j)}\}_{1\le j\le n})=Diag(\sigma^{(1)}, \dots, \sigma^{(n)})$, see Definition~\ref{def:svd} for SVD.  For $\nu\ge 0$, we define SV shrinkage operator $\mathcal{D}_{\nu}$ of the matrix  $G$ as
\begin{equation}\label{eqn:shrink}
\begin{aligned}
\mathcal{D}_{\nu}(G)&=U\mathcal{F}(\Sigma)V^T, \text{where}\\
\mathcal{F}(\Sigma)&=\mathrm{Diag}\left(\{\max(\sigma^{(j)}-\nu,0)\}_{1\le j\le n}\right).\\
\end{aligned}
\end{equation} 
\end{definition}

We use Theorem~\ref{trm:shrink} \cite{hu2012fast, Cai2010}, to get a closed-form solution for \eqref{eqn:solForK}. Note that, $G$ and $\nu$ in this theorem are equivalent to $K_m-(1/\rho^{\zeta}_m)\zeta_m$ and $1/\rho^{\zeta}_m$ in \eqref{eqn:solForK}, respectively. This theorem states an alternative implementation of a vital minimization problem using shrinkage of SVs. When $\mathcal{L}$ in \eqref{eqn:ALM} convergences to the true solution, the conditions $L=K$ and $\zeta=0$ are satisfied. Thus, imposing $\mathrm{Gram}(G)\succeq 0$ is an approximation for imposing $\mathrm{Gram}(L)\succeq 0$. Theorem~\ref{trm:shrink} along with this approximation provides
\begin{equation}\label{eqn:solForK1}
\begin{aligned}
L_{m+1} = \{\mathcal{D}_{\nu}(G_m) \vert \mathrm{Gram}(G_m)\succeq 0\}; \\
G_m = K_m-\frac{1}{\rho^{\zeta}_m}\zeta_m \ \text{and} \ \nu = \frac{1}{\rho^{\zeta}_m}.\\
\end{aligned}
\end{equation}
\begin{theorem} \label{trm:shrink} \cite{hu2012fast, Cai2010},
For any $\nu\ge0$ and $G\in \mathbb{R}^{n\times n}$,  
\begin{equation}\label{eqn:shrinkage}
\mathcal{D}_{\nu}(G)=\underset{L}{\argmin} \ \left[\nu\|L\|_*+\frac{1}{2}\|L-G\|^2_F\right].
\end{equation} 
Note that, $\|\cdot\|_*$ and $\|\cdot\|_F$ are defined in Definitions~\ref{def:nucl} and \ref{def:frob}, respectively.
\end{theorem}

\begin{proof}
See the proof of Theorem~2.1 in ~\cite{Cai2010}.
\end{proof}

The shrinkage operator in Definition~\ref{def:shrink} shrinks SVs of a matrix $G$, similar to $G_m$ in \eqref{eqn:solForK1}, by the value of $\nu \ge 0$ . These shrunk SVs along with unitary matrices, say $U_m$ and $V_m$, from the SVD are used to compute $\mathcal{D}_{\nu}(G_m)$ of \eqref{eqn:solForK1}. To incorporate the condition $\mathrm{Gram}(G_m)\succeq 0$ of \eqref{eqn:solForK1} we perform a modification to this definition under two steps 1) we compute the EVD of $\mathrm{Gram}(G_m)$ instead of the SVD of $G_m$ and 2) we assign zeros for the EVs that are either negative or less than some value $\nu$. 

Step 1: we transform $G_m$ to its Gramian using the matrix version of \eqref{eqn:double_centering} given in \eqref{eqn:solForK2}. Let, $\rm{diag}(G_m)$ denotes a vector formed by the diagonal of $G_m$ and $\bs{e}=[1, \dots, 1]^T$ is a column vector of ones. Note that, $\tilde{G}_m=G_m-\mathrm{Diag}(\mathrm{diag}(G_m))$ represents a matrix formed by replacing the diagonal of $G_m$ with zeros. Thus, \cite{gower1985applications} states that
\begin{equation}\label{eqn:solForK2}
\mathrm{Gram}(G_m)=-\frac{1}{2}\left(I-\frac{1}{n}\bs{e}\bs{e}^T\right)\tilde{G}_m\left(I-\frac{1}{n}\bs{e}\bs{e}^T\right).
\end{equation}
We compute EVD of $\mathrm{Gram}(G_m)$ using Definition~\ref{def:evd} such that
\begin{equation}\label{eqn:unitrary_EVD}
\mathrm{Gram}(G_m) = U_m \Lambda_m U^{-1}_m; \ \Lambda_m = \mathrm{Diag}\left(\left\{\lambda^{(j)}_m\right\}_{1\le j \le n}\right),
\end{equation}
where $\lambda^{(j)}_m$ represents the $j$-th eigenvalue of $\mathrm{Gram}(G_m)$ at the $m$-th iteration. 

Step 2: to impose $\mathrm{Gram}(G_m)\succeq 0$, we define a new shrinkage operator $S_{\nu}$. This shrinkage operator guarantees that the shrunk EVs of $\mathrm{Gram}(G_m)$ are non-negative which will assure $\mathrm{Gram}(G_m)\succeq 0$. We are also interested in eliminating EVs that are positive and small since they represent non-prominent features of the dataset including noise. For that, we adjust the shrinkage function $\mathcal{F}$ \eqref{eqn:shrink} to impose that shrunk EVs are not less than $1/\rho^{\zeta}_m$. Thus, the new shrinkage function, denoted by $\mathcal{F}^*$, is 
\begin{equation}\label{eqn:Newshrinkage}
\begin{aligned}
\mathcal{F}^*(\Lambda_m)&=\mathrm{Diag}\left(\{\lambda^{*(j)}_m\}_{1\le j\le n}\right);\\
\lambda^{*(j)}_m &= 
\left\{
\begin{array}{ll}
      \lambda^{(j)}_m-\nu^*, & \mathrm{if} \ \lambda^{(j)}_m-\nu^*>0, \\     
      0, & \mathrm{otherwise}, \\
\end{array} 
\right.
\end{aligned}
\end{equation}
where $\nu^*=(1/\rho^{\zeta}_m)-\min(\lambda^{(1)}_m, \dots, \lambda^{(n)}_m, 0)$. We use these shrunk EVs of $\mathrm{Gram}(G_m)$ to formulate the shrunk Gramian, denoted by $S_{\nu}(\mathrm{Gram}(G_m))$, which is PSD. For that, we use the matrix $U_m$ in \eqref{eqn:unitrary_EVD} as
\begin{equation}\label{eqn:psdgram}
S_{\nu}(\mathrm{Gram}(G_m))=U_m \mathcal{F}^*(\Lambda_m)U^{-1}_m.
\end{equation}
Now, we transform $S_{\nu}(\mathrm{Gram}(G_m))$ back to its corresponding $G_m$, denoted by $G^*_m$, using the reverse formula of \eqref{eqn:solForK1} as
\begin{equation}\label{eqn:psdG}
\begin{aligned}
G^*_m=\mathrm{Diag}(S_{\nu}(\mathrm{Gram}(G_m)))\bs{e}^T\\
+\bs{e}[\mathrm{Diag}(S_{\nu}(\mathrm{Gram}(G_m)))]^T\\
-2S_{\nu}(\mathrm{Gram}(G_m)),\\
\end{aligned}
\end{equation}
as in \cite{gower1985applications}.

The above two steps force the recovered matrix to be a distance matrix and retain the EVs representing prominent features of the dataset. Finally, the solution for $L$ of \eqref{eqn:solForK} is given as
\begin{equation}
L_{m+1}=G^*_m.
\end{equation}

\subsubsection{Minimization of ALF with respect to $K$}\label{sec:min_ALF_with_K}
Here, we minimize $\mathcal{L}$ in \eqref{eqn:ALM} with respect to $K$ at the $m$-th time step and obtain a new solution for $K$, denoted by $K_{m+1}$. Since we have minimized this Lagrangian with respect to $L$ in Sec.~\ref{sec:min_ALF_with_L}, we use the current solution of $L$, i.e. $L_{m+1}$, and partially update the Lagrangian. Thus,
\begin{equation}
\begin{aligned}\label{eqn:solForL_old}
K_{m+1} &= \underset{K}{\text{arg min}} \ \mathcal{L}(L_{m+1},K,\zeta_m,\eta_m,\rho^{\zeta}_m,\rho^{\eta}_m\vert \mathrm{Gram}(L)\succeq 0)\\
& = \underset{K}{\text{arg min}}\left[-\mathrm{tr}(U_r K V^T_r)+\langle \zeta_m,L_{m+1}-K \rangle\right.\\
& \hspace{.4cm}\left.+ \frac{\rho^{\zeta}_m}{2}\|L_{m+1}-K\|^2_F+ \langle\eta_m,E_m\rangle+\frac{\rho^{\eta}_m}{2}\|E_m\|^2_F\right].
\end{aligned}
\end{equation}
Again, we ignored the terms independent of $K$ in \eqref{eqn:solForL_old}. Using the relationship $\|A\|^2_F=\langle A, A \rangle$, where $A$ is an arbitrary matrix, and using some identities of the inner product, \eqref{eqn:solForL_old} is simplified as
\begin{equation}
\begin{aligned}\label{eqn:solForL}
K_{m+1} &= \underset{K}{\text{arg min}}\left[\frac{\rho^{\zeta}_m}{2}\left\|(L_{m+1}-K)+\frac{1}{\rho^{\zeta}_m}(U^T_r V_r+\zeta_m)\right\|^2_F\right.\\
&  \hspace{4.6cm}\left.+ \frac{\rho^{\eta}_m}{2}\left\|E_m+\frac{\eta_m}{\rho^{\eta}_m}\right\|^2_F\right].
\end{aligned}
\end{equation}

We differentiate the function inside arg min in \eqref{eqn:solForL} with respect to $K$ and solve it for $K$ as
\begin{equation}\label{eqn:solForL1}
K_{m+1}=\frac{1}{\rho^{\zeta}_m}\left[\rho^{\zeta}_m L_{m+1}+\left(U^T_r V_r+\zeta_m\right)
-(\rho^{\eta}_m E+\eta_m)E'\right].
\end{equation}
Note that, $K$ of the function $E=E(D^l, D^u, K)$ varies through iterations (e.g., as in Fig.~\ref{fig:funE} and \eqref{eqn:optRouting2}) whereas $D^l$ and $D^u$ are fixed through iterations, for a given problem. Thus, the derivative of the function $E$ in \eqref{eqn:optRouting2}, denoted by $E'=\left[E'_{ij}\right]_{n\times n}$, is 
\begin{equation}
E'_{ij} = 
\left\{
\begin{array}{ll}
	1, & \mathrm{if} \ D^l_{ij}>K_{ij}>D^u_{ij}, \\      	
      	0, & \mathrm{otherwise}. \\
\end{array} 
\right.
\end{equation}

\subsubsection{Minimization of ALF with respect to $\zeta$}
Now, we partially update $\mathcal{L}$ in \eqref{eqn:ALM} with the current solutions $L_{m+1}$ and $K_{m+1}$, and minimize $\mathcal{L}$ with respect to $\zeta$ to obtain its new solution $\zeta_{m+1}$ as
\begin{equation}
\begin{aligned}\label{eqn:solForLambda}
\zeta_{m+1} &= \underset{\zeta}{\text{arg min}} \ \mathcal{L}\left(
\begin{array}{c}
	K_{m+1},L_{m+1},\zeta,\eta_m,\rho^{\zeta}_m,\rho^{\eta}_m\\
	\vert \ \  \mathrm{Gram}(L)\succeq 0\\
\end{array} 
\right),\\
& = \underset{\zeta}{\text{arg min}}\left[\langle \zeta,L_{m+1}-K_{m+1} \rangle\right].\\
\end{aligned}
\end{equation}
Consider that we use the symbol $|$ for ``such that" in \eqref{eqn:solForLambda} and thereafter. Differentiation of \eqref{eqn:solForLambda} with respect to $\zeta$ gives 
\begin{equation}
\zeta_{m+1}=\zeta_m+\rho^{\zeta}_m (L_{m+1}-K_{m+1}).
\end{equation}

\subsubsection{Minimization of ALF with respect to $\eta$}
After partially updating $\mathcal{L}$ in \eqref{eqn:ALM} with the current solutions $L_{m+1}$, $K_{m+1}$, and $\zeta_{m+1}$, we obtain
\begin{equation}
\begin{aligned}\label{eqn:solForGamma}
\eta_{m+1} &= \underset{\eta}{\text{arg min}} \ \mathcal{L}\left(
\begin{array}{c}
	K_{m+1},L_{m+1},\zeta_{m+1},\eta,\rho^{\zeta}_m,\rho^{\eta}_m\\
	\vert \ \  \mathrm{Gram}(L)\succeq 0\\
\end{array}
\right),\\
& = \underset{\eta}{\text{arg min}}\left[\langle \eta,E\rangle\right].\\
\end{aligned}
\end{equation}
Differentiation of \eqref{eqn:solForGamma} with respect to $\eta$ gives 
\begin{equation}
\eta_{m+1}=\eta_m+\rho^{\eta}_m E_{m+1}.
\end{equation}

We observed in the testing process that BMC converges fast and is numerically stable when the ADMM solver emphasizes the minimization of the objective function in the early iterations and then it emphasizes on satisfying the constraints in the later iterations. Note that, according to \eqref{eqn:ALM}, we can control the trade-off between satisfying the objective function and satisfying the constraints by changing the values of the parameters $\rho^{\zeta}_m$ and $\rho^{\eta}_m$. Thus, to achieve those computational advantages, we set variable values for the parameters $\rho^{\zeta}_m$ and $\rho^{\eta}_m$. For that, we introduce a new parameter $\rho\ge 1$ that we use to increase the value of the parameters $\rho^{\zeta}_m$ and $\rho^{\eta}_m$ through iterations such that 
\begin{equation}
\begin{aligned}\label{eqn:gammaGamma}
\rho^{\zeta}_{m+1} = \rho \ \rho^{\zeta}_m, \\
\rho^{\eta}_{m+1} = \rho \ \rho^{\eta}_m.
\end{aligned}
\end{equation}

The ADMM iterative scheme of BMC requires user input initial guesses $L_1\in\mathbb{R}^{n\times n}$, $K_1\in\mathbb{R}^{n\times n}$, $\zeta_1\in\mathbb{R}^{n\times n}$, $\eta_1\in\mathbb{R}^{n\times n}$, $\rho^{\zeta}_1\in\mathbb{R}$, and $\rho^{\eta}_1\in\mathbb{R}$. We will explain setting these initial guesses in Sec.~\ref{sec:per_analysis}.

\section{Performance analysis}\label{sec:per_analysis}
Here, we validate BMC using three representative examples based on three datasets: one synthetic dataset that we sampled from a semi-cylinder and two real-life datasets, face images \cite{tenenbaum2016facedata} and handwritten digits \cite{lecun2016the}. We compare the results of BMC with that of five DR methods from the literature, namely ISOMAP, DM, LE, LLE, and HLLE. We provide a brief description of each method below.

\begin{itemize}
\item \emph{Isometric Feature Mapping} (ISOMAP) \cite{Tenenbaum2000}, has one parameter to specify the neighborhood size, say $\delta$. ISOMAP turns an input dataset of points into a graph structure, first, by treating points as nodes, and then, by creating edges between each point and that point's $\delta$ nearest points. The weight of an edge is considered to be the ED between the adjacent two points. ISOMAP approximates the MD between any two points in the dataset as the length of the shortest path between them. Finally, this method transforms the MD matrix into an inner product matrix, as in \eqref{eqn:double_centering}, and the SVs of the inner product matrix are used to produce the latent variables. 

\item The DR method \emph{Diffusion Maps} (DM) \cite{Lafon2006a, Coifman2006b}, creates a graph that is similar to that of ISOMAP which has weights of the edges computed using the Gaussian kernel function with a user input variance $\gamma$. DM defines Markov random walks on this graph for a $t$ (an input parameter) number of time-steps to compute a proximity measure of the data points. These measures are transformed into diffusion distances and the latent variables are computed using the spectral decomposition on the diffusion distance matrix. 

\item \emph{Laplacian Eigenmaps} (LE) \cite{Belkin2001},  computes a low-dimensional representation of the data in which the distances between a datapoint and its $\delta$ nearest neighbors are minimized. This is done in a weighted manner using a Gaussian kernel function with a user input variance $\gamma$ such that the closer the neighbor, the more weight that contributes to the cost function. 

\item \emph{Local Linear Embedding} (LLE) \cite{Roweis2000}, constructs the local properties of the manifold underlying the data by formulating the high-dimensional datapoints as a linear combination of their $\delta$ (an input parameter) nearest neighbors with some weights. LLE attempts to retain the weights of the linear combinations in the low-dimensional reconstruction as close as possible to that of high-dimensional representation. 

\item \emph{Hessian Local Linear Embedding} (HLLE) \cite{Coifman2006b}, is a variant of LLE that minimizes the curviness of the high-dimensional manifold under the constraint that the low-dimensional data representation is locally isometric. This is done by an eigenanalysis of a matrix that describes the curviness, measured by means of the local Hessian, of the manifold around $\delta$ (an input parameter) nearest neighbors at each data point.
\end{itemize}

\subsection{A semi-cylinder having a hollow region}\label{sec:semi-cylinder}
We analyze the performance of BMC by embedding a dataset sampled from a semi-cylinder, a 2-D manifold, having a hollow region. First, we visually compare the embedding performed by BMC with the embedding performed by five DR methods, and then we analyze the SV spectrums of the distance matrices of the manifolds recovered by BMC and the other five DR methods.

We sample 500 points from a hollowed semi-cylinder using 
\begin{equation}\label{eqn:semi-sphere}
y_1 = 4 \cos(\theta), \ \ y_2 = 4 \sin(\theta), \ \ y_3 = z,
\end{equation}
where $\theta \in \mathbb{U}[[0,\pi/3] \cup [2\pi/3,\pi]]$ and $z \in \mathbb{U}[[0,3] \cup [7,10]]$ as in Fig.~\ref{fig:cylinEmb}(a). Here, $\mathbb{U}[a,b]$ denotes an uniform probability distribution from $a$ to $b$. 

Since BMC inputs a lower bound ($D^l$) and a upper bound ($D^u$) of the squared MD matrix, we compute them using \eqref{eqn:lowerBound} and \eqref{eqn:upperBound}. Even though the squared MDs of points in the same cluster of this dataset could easily be computed, we treat these MDs as unknowns to make this problem little harder. This is numerically implemented by setting unequal values for $\alpha^l_{ij}$ and $\alpha^u_{ij}$ for each $ij$ as explained in Sec.~\ref{sec:mc}. Moreover, for all $ij$, we set $\alpha^l_{ij}$ to a small value, say .1, and set $\alpha^u_{ij}$ to a large value, say10, in \eqref{eqn:lowerBound} and \eqref{eqn:upperBound}. We set wide gaps between upper and lower bounds to make sure that their true manifold distances lie between them. Narrow gaps between lower and upper bounds expedite the convergence of BMC; however, the construction of valid narrow gaps is to be done in a separate research that we will summarize in Sec.~\ref{sec:conclusion}. 

We set the initial guesses  arbitrarily as $L_1=\ $$K_1= \ $$\zeta_1= \ $$\eta_1= \ $$.8D^l+.2D^u$. Consider that any matrix between $D^l$ and $D^u$ would work as an initial guess. We also set $\rho^{\zeta}_1=\ $$\rho^{\eta}_1=.05$ and $\rho=1.01$. As we explained in Sec.~\ref{sec:admm}, setting $\rho^{\zeta}_1$ and $\ $$\rho^{\eta}_1$ to a value smaller than one and setting $\rho$ to a value slightly greater than one (this increases the parameters $\rho^{\zeta}_m$ and $\rho^{\eta}_m$ in \eqref{eqn:ALM} through iterations) assures numerical stability and fast convergence of BMC. We run BMC with $r=4$ for 500 iterations. We set $r=4$, since Theorem~\ref{thm:rank} states that the squared distance matrix of a 2-D manifold has four non-zero SVs.

Now, we perform the embedding of the semi-cylinder using the recovered distance matrix $L$. For that, we compute the Gramian of $L=[L_{ij}]_{n \times n}$ using \eqref{eqn:double_centering} and then perform SVD of $\mathrm{Gram}(L)$ using 
\begin{equation}\label{eqn:evd}
\mathrm{Gram}(L)=U\Sigma V^T,
\end{equation}
where $U$ and $V$ denote unitary matrices, and $\Sigma$ denotes a diagonal matrix of SVs. $p$-D latent variables (collectively called an embedding ) of a high-dimensional dataset is given by
\begin{equation}\label{eqn:latent_var}
\hat{X}=I_{p\times n}\Sigma^{1/2}V^T,
\end{equation}
where $I_{p\times n}$ is a matrix representing the first $p$-rows of the identity matrix $I_{n\times n}$ \cite{lee2004nonlinear}. We set $p=2$ to get 2-D embedding of this dataset as in Fig.~\ref{fig:cylinEmb}(b).

We also run the DR methods ISOMAP with $\delta=12$ as in Fig.~\ref{fig:cylinEmb}(c), DM with $\gamma=1$ and $t=10$ as in Fig.~\ref{fig:cylinEmb}(d), LE with $\delta=12$ and $\gamma=1$ as in Fig.~\ref{fig:cylinEmb}(e), LLE with $\delta=12$ as in Fig.~\ref{fig:cylinEmb}(f), and HLLE with $\delta=12$ as in Fig.~\ref{fig:cylinEmb}(g), over this dataset and compute 2-D embeddings. We observe that only BMC and DM preserve the geometry of the data.   However, we will demonstrate on the real-world datasets that BMC provides substantial benefits over DM in more challenging problems.

\begin{SCfigure*}
  \centering
  \caption{2-D embeddings of a cylindrical dataset having a hollow region. (a) A dataset of 500 points is embedded in 2-D using (b) BMC with $r=4$, $\rho=1.01$, $\rho^{\zeta}_1=\rho^{\eta}_1=.05$, and $\alpha^l_{ij}=.1$ and $\alpha^u_{ij}=10$ for all $ij$; (c) ISOMAP with $\delta=12$; (d) DM with $\gamma=1$ and $t=10$; (e) LE with $\delta=12$ and $\gamma=1$; (f) LLE with $\delta=12$; and (g) HLLE with $\delta=12$.}
  \includegraphics[width=0.75\textwidth]{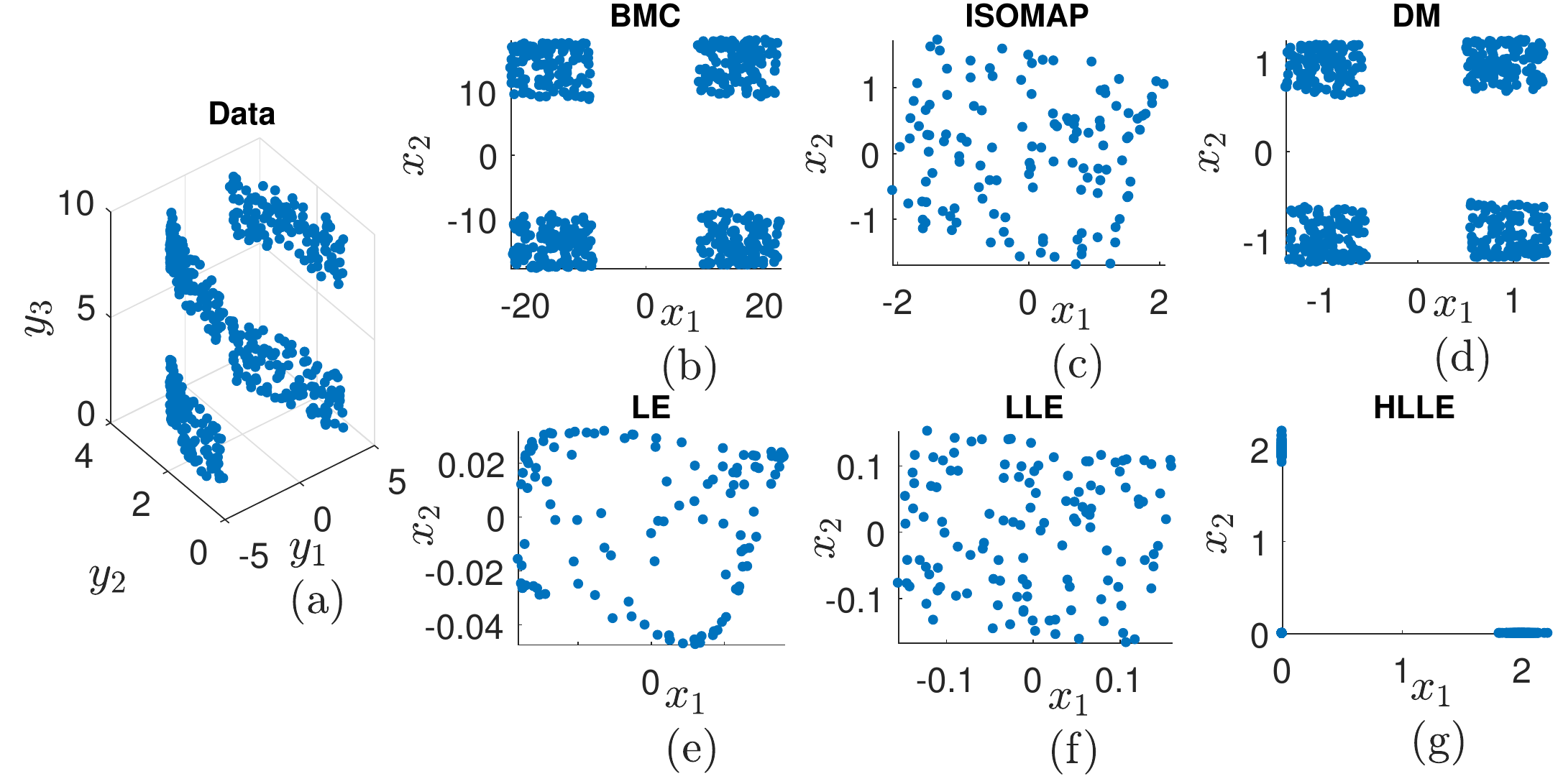}
\label{fig:cylinEmb}
\end{SCfigure*}

Now, we analyze the SV spectrum of BMC and that of the other five DR methods. First, we compute the 10 largest SVs of the recovered matrix $L$. Computations of the SV spectrums of the other embeddings are performed under two steps. First, we embed the dataset in 10-D using ISOMAP, DM, LE, LLE, and HLLE. Then, we compute the ED matrix of the embedding of each of the five methods and then compute the SV spectrum for each of the distance matrices. As the ground-truth, we compute the 10 largest SVs of the squared MD matrix of the data (MD takes curvature into consideration that can easily be computed since this data is sampled from cylinder). For the SVs $\{\sigma^{(j)}| j = 1, \dots, 10\}$ of the squared distance matrix of an embedding, we define the percentage of normalized SVs as
\begin{equation}
\hat{\sigma}^{(j)}=100\frac{\sigma^{(j)}}{\sum^{10}_{j=1} \sigma^{(j)}}.
\end{equation}
Fig.~\ref{fig:SVvsDr} shows $\ln(\hat{\sigma}^{(j)})$; $j=1,\dots,10$ of all the DR methods. We observe in this figure that the first four $\hat{\sigma}^{(j)}$'s of BMC are nonzero and the rest of them are approximately zero which is similar to that of the manifold (ground-truth). However, none of the DR methods except BMC have zeros after $\hat{\sigma}^{(4)}$. Thus, except BMC all the other methods do not recover the solution guaranteed by Theorem \ref{thm:rank} and that might lead to an embedding which is not representative to the dataset.

\begin{figure}[htp]
    \centering
    \includegraphics[width=3.3in]{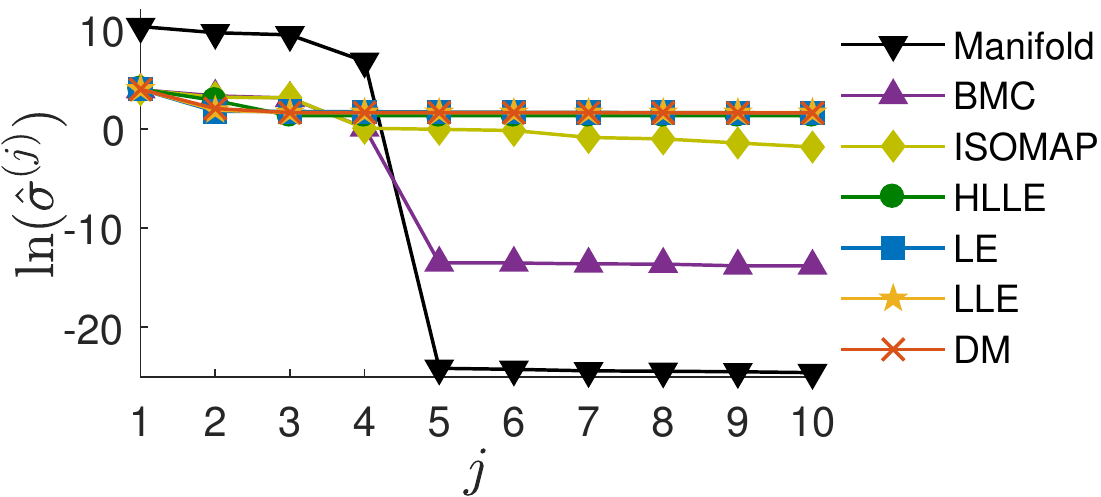}
    \caption{Semi-log plots of percentage of normalized SVs of the squared manifold distance matrix and that of the squared distance matrices of the embeddings of BMC and the other DR methods.}
\label{fig:SVvsDr}
\end{figure}

\subsection{Handwritten digits}\label{sec:handwritten}
Now, we embed a sample of handwritten digits obtained from the Modified National Institute of Standards and Technology (MNIST) database \cite{lecun2016the}. It is well known \cite{lecun2016the} that the low-dimensional representations MNIST handwritten digits show rich clustering which we will analyze using BMC and the other DR methods in this section. The MNIST database has 70,000 images of handwritten digits 0, 1, $\dots$, 9 from which we sample (uniform random sampling) 30 images of arbitrary digits of each 0, 1, 3, and 4, (the image-set size is 120 images of digits) as in Fig~\ref{fig:mnist}(a). 

Each image in the MNIST database is $28 \times 28$ dimensions. Since each pixel represents one dimension, each image is a point in a 784-D observable space \cite{gajamannage2018nonlinear}. We reshape each such image into a row vector and concatenate these 120 row vectors (note that, this image-set has 120 images) vertically to produce the dataset. We are not aware of information about the data such as hollow regions or clusters like those in the semi-cylinder dataset in Sec.~\ref{sec:semi-cylinder}. Thus, we treat each point of this dataset as a separate cluster, and compute the lower bound ($D^l$) and the upper bound ($D^u$) of the MDs using \eqref{eqn:lowerBound} and \eqref{eqn:upperBound} with $\alpha^l_{ij}=.1$ and $\alpha^u_{ij}=10$ for all $ij$. Consider that computing the bounds $D^l$ and $D^u$, and setting the parameters $\alpha^l_{ij}$ and $\alpha^u_{ij}$ are done in a similar manner as those in the semi-cylinder example. Here, we set the initial guesses as the same as those in the semi-cylinder example. Since we are not aware of the dimensionality of the manifold underlying the image-set, we are unable to set $r$ based on Theorem~\ref{thm:rank}. However, we set $r=4$ for BMC to generate a 2-D embedding which we can easily visualize \footnote{Other choices of $r$ are, of course, possible.}. We run BMC with $\rho=1.02$ (setting $\rho$ is based on the same argument as that in the semi-cylinder example) for 800 iterations and compute the recovered squared distance matrix $L$. We compute the Gramian of $L$ using \eqref{eqn:double_centering} and perform SVD of $\mathrm{Gram}(L)$ using \eqref{eqn:evd}. We estimate $2$-D embedding of the this dataset by setting $p=2$ in \eqref{eqn:latent_var}.

We run DR methods ISOMAP with $\delta=12$ as in Fig.~\ref{fig:mnist}(c), DM with $\gamma=1$ and $t=10$ as in Fig.~\ref{fig:mnist}(d), LE with $\delta=12$ and $\gamma=1$ as in Fig.~\ref{fig:mnist}(e), LLE with $\delta=12$ as in Fig.~\ref{fig:mnist}(f), and HLLE with $\delta=12$ as in Fig.~\ref{fig:mnist}(g), over this dataset and compute 2-D embeddings. We see in Fig.~\ref{fig:mnist} that while BMC, ISOMAP, and LE assure some level of clustering of digits, DM, LLE, and HLLE fail to preserve the clustering. In order to learn the clustering ability in a qualitative manner, we numerically analyze the misclustering of digits using k-Means Clustering (KMC) method available in \cite{lloyd1982least}. KMC is a data-partitioning algorithm that assigns given observations, say $n$, to exactly one of the given clusters, say $k$ in total, defined by centroids where $k$ is chosen before the algorithm starts. We set $k=4$ in KMC as the dataset consists of four digits and find the clustering for each digit. If a digit is misclustered, then we add one penalty towards the error. The clustering error, denoted by $\mathcal{E}_c$, is defined as the percentage of misclustering in the dataset. We compute the clustering error for each embedding and display them at the titles of Figs.~\ref{fig:mnist}(b-g). The clustering errors justify that BMC performs well in terms of preserving the clustering in contrast to that of the other DR methods. 

\begin{SCfigure*}
  \centering
  \caption{2-D embeddings and clustering errors of a image-set of 120 handwritten digits. (a) A sample of 60 images of the image-set. This image-set is embedded in 2-D using (b) BMC with $r=4$, $\rho=1.02$, $\rho^{\zeta}_1=\rho^{\eta}_1=.1$, and $\alpha^l_{ij}=.1$ and $\alpha^u_{ij}=10$ for all $ij$; (c) ISOMAP with $\delta=12$; (d) DM with $\gamma=1$ and $t=10$; (e) LE with $\delta=12$ and $\gamma=1$; (f) LLE with $\delta=12$; and (g) HLLE with $\delta=12$. The title at each embedding also display its clustering error $\mathcal{E}_c$.}
  \includegraphics[width=0.75\textwidth]{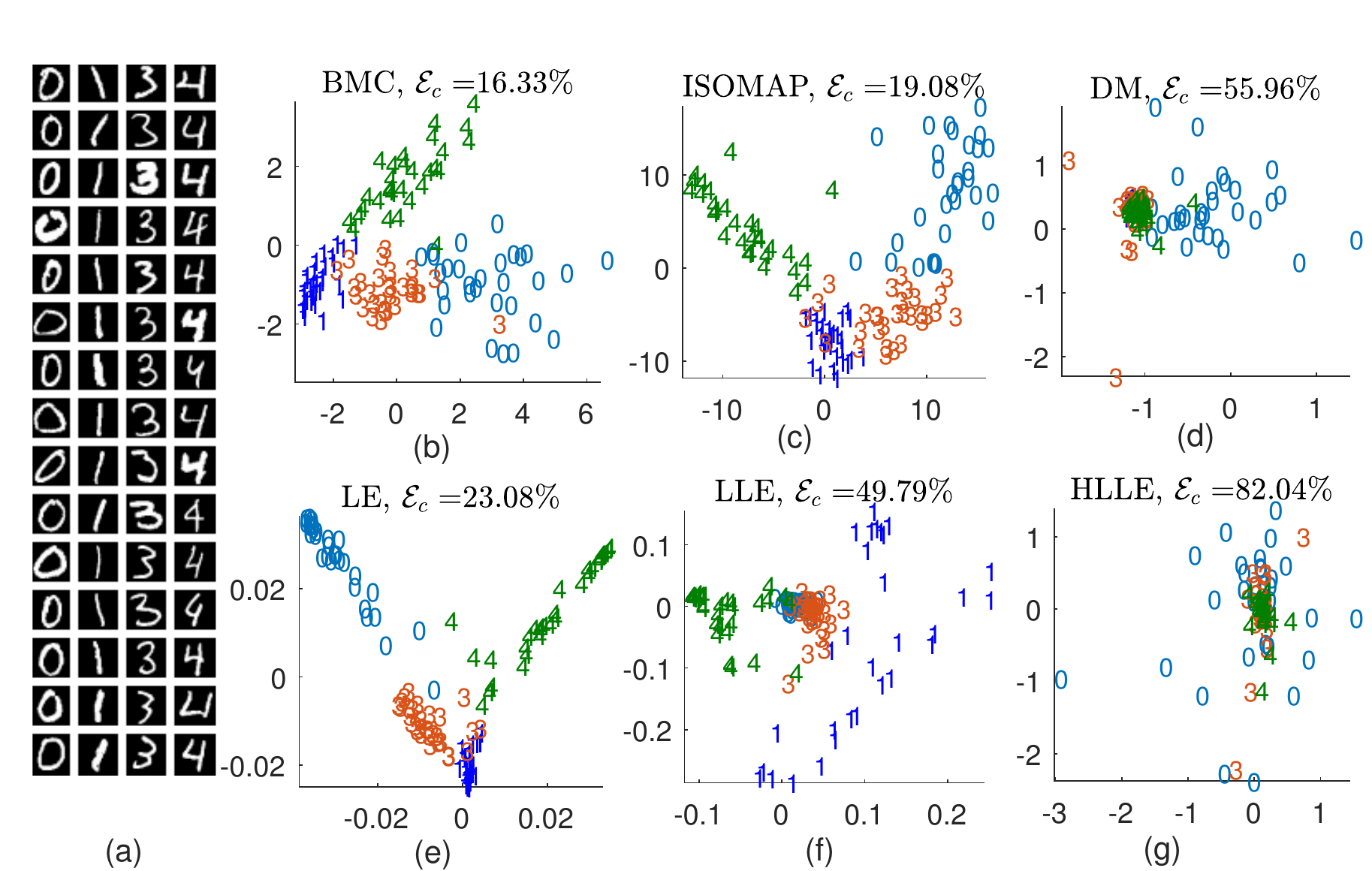}
\label{fig:mnist}
\end{SCfigure*}

\subsection{Face images}\label{sec:faceimages}
 In this section, we embed a noisy sample of face images obtained from \cite{tenenbaum2016facedata} and learn BMC's ability to preserve the underlying topology. This face image database consists of 600 face images of $64\times 64$ dimensions having three variations:  left-right orientation of light change; left-right pose change; and up-down pose change. Note that, each image of this database is a point in a 4096-D observable space.

We sample (uniform random sampling) 300 images from this database and contaminate them with a Gaussian noise of variance one as in Fig~\ref{fig:face_images}(a). Similar to the handwritten digit example, we reshape each of these image into a row vector and concatenate all such 300 rows vertically to form the dataset. Here, we are not aware of geographical information about this dataset such as hollow regions or clusters like those in the semi-cylinder dataset in Sec.~\ref{sec:semi-cylinder}. Thus, we treat each point as a separate cluster, and compute $D^l$ and $D^u$ of the MDs using \eqref{eqn:lowerBound} and \eqref{eqn:upperBound} with $\alpha^l_{ij}=.1$ and $\alpha^u_{ij}=10$ for all $ij$ in a similar manner as those in the semi-cylinder example. We also  set the initial guesses as same as those in the semi-cylinder example. Since each image has three variables (2-D pose and 1-D orientation of light), the underlying manifold is 3-D which is incorporated by setting $r=5$ according to Theorem~\ref{thm:rank}. We run BMC with $\rho=1.04$ (setting $\rho$ is based on the same argument as that in the semi-cylinder example) for 800 iterations and compute the recovered squared distance matrix $L$. We compute the Gramian of $L$ using \eqref{eqn:double_centering} and perform SVD of $\mathrm{Gram}(L)$ using \eqref{eqn:evd}. We estimate $2$-D embedding of this dataset by setting $p=2$ in \eqref{eqn:latent_var}, see Fig.~\ref{fig:face_images}(b). 

We run the DR methods ISOMAP as in Fig.~\ref{fig:face_images}(c), DM as in Fig.~\ref{fig:face_images}(d), LE as in Fig.~\ref{fig:face_images}(e), LLE as in Fig.~\ref{fig:face_images}(f), and HLLE as in Fig.~\ref{fig:face_images}(g), over this dataset with the parameter values the same as those in the handwritten example and compute 2-D embeddings. We see in Fig.~\ref{fig:face_images} that while HLLE performs a weak embedding (most of the images are projected onto a small area), the other methods demonstrate some level of good embedding. 

\begin{SCfigure*}
  \centering
  \caption{2-D embeddings of a image-set of 300 faces. (a) A sample of 60 face images of the image-set whereas columns from left to right indicate left-right light change, left-right pose change, up-down pose change, and a noisy version of images, respectively. This image-set is embedded in 2-D using (b) BMC with $r=5$, $\rho=1.04$, $\rho^{\zeta}_1=\ $ $\rho^{\eta}_1=.1$, and $\alpha^l_{ij}=.1$ and $\alpha^u_{ij}=10$ for all $ij$; (c) ISOMAP with $\delta=12$; (d) DM with $\gamma=1$ and $t=10$; (e) LE with $\delta=12$ and $\gamma=1$; (f) LLE with $\delta=12$; and (g) HLLE with $\delta=12$. The snapshots show the appearance of some faces on this manifold.}
  \includegraphics[width=0.7\textwidth]{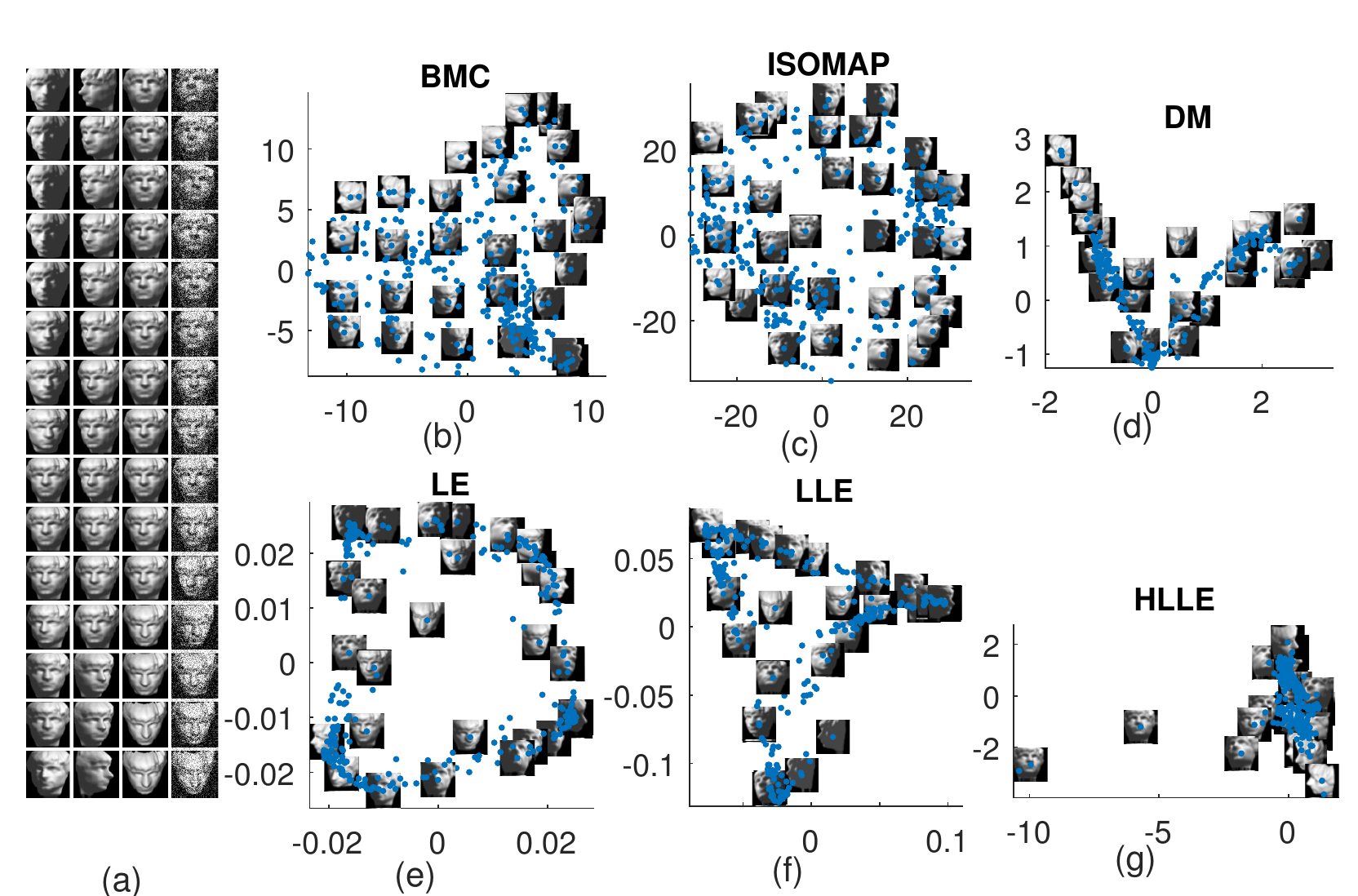}
\label{fig:face_images}
\end{SCfigure*}

In order to quantitatively analyze the embedding performance, we compute neighborhood preserving error \cite{gajamannage2015dimensionality}, denoted by $\mathcal{E}_{\delta'}$. For that, first, we make an adjacency distance matrix $A^d = [A^d_{ij}]_{n\times n}$ of the dataset such that $A^d_{ij}=1$, if the $j$-th point is one of $\delta'$ neighbors (an input parameter) of the $i$-th point; or $A^d_{ij}=0$ otherwise. Note that, the parameter $\delta'$ is similar to $\delta$ of ISOMAP in Sec.~\ref{sec:per_analysis}. Similarly, we make an adjacency distance matrix $A^e = [A^e_{ij}]_{n\times n}$ of an embedding of interest. The neighborhood preserving error is defined as the percentage of normalized cumulative absolute difference between the adjacency distance matrices $A^d$ and $A^e$,
\begin{equation}\label{eqn:adj_error}
\mathcal{E}_{\delta'}=\frac{100}{\delta' (n-1)}\sum_{i,j=1}^{n}\big\vert A^d_{ij}-A^e_{ij}\big\vert,
\end{equation}
where $\delta'$ is the size of the neighborhood of interest\cite{gajamannage2015dimensionality}. Fig.~\ref{fig:face_image_err} shows $\mathcal{E}_{\delta'}$ versus $\delta'=1,\dots,20$ for both BMC and the other DR methods. We observe in this figure that BMC preserves the neighborhood than that of the other DR methods for all the neighborhood sizes 1--20. Moreover, the neighborhood preserving ability of BMC is better than that of the other methods when the neighborhood size is small.

\begin{figure}[h]
	\centering
        	\includegraphics[width=3.2in]{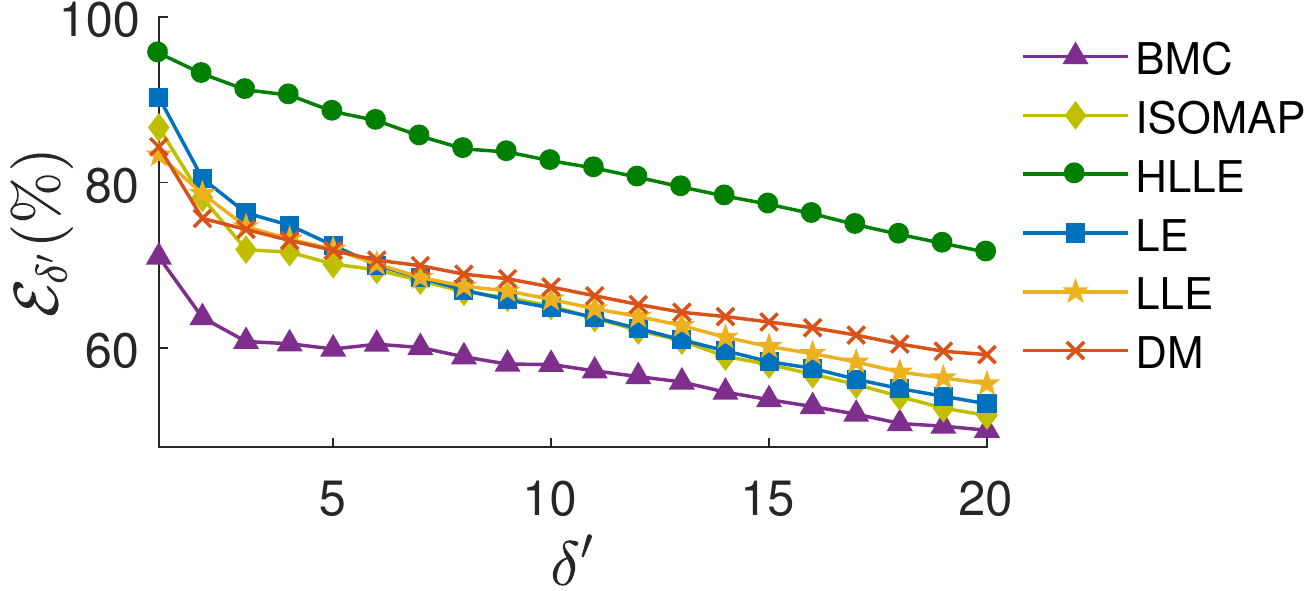}
        	\caption{Neighborhood preserving ability of BMC and the other DR methods versus different neighborhood sizes.}
	\label{fig:face_image_err}
\end{figure}

\section{Conclusion}\label{sec:conclusion}
In this paper, we presented an MC scheme that 1) reveals a distance matrix for a partially observed matrix; 2) employs truncated rank; and 3) inputs the bounds for each entry of the partially observed matrix so that the recovered matrix should satisfy those bounds. This MC scheme inherits truncated nuclear norm regularization for truncated rank from \cite{hu2012fast} and inherits recovery of the full matrix using bounds from \cite{paffenroth2013space}. In order to assure that the recovered matrix is a distance matrix, we imposed that the recovered matrix is PSD using a novel approach as stated in Sec.~\ref{sec:admm}. Fig.~\ref{fig:SVvsDr} justifies that BMC is robust of revealing an accurate low-rank approximation of a partially observed distance matrix that has the same rank as that of the manifold.

We observed in Sec.~\ref{sec:handwritten} that BMC attains the least clustering error among all the DR methods of interest. Thus, it justifies that BMC provides superior clustering when it embeds a cluster rich dataset like MNIST. However, the clustering error of  BMC was as high as $16.33\%$ since we performed a 2-D embedding without knowing the true dimensionality of the dataset. Thus, before performing this analysis, the true dimensionality of the manifold underlying the data should be revealed. In the future, we will run BMC with $r=3,\ $$\dots, m$ for some $m$ and then analyze the SV spectrums of the recovered matrices to learn the true manifold dimensionality. 

The face image example in Sec.~\ref{sec:faceimages} demonstrated BMC's capability of embedding a real-life dataset that is contaminated with a high level of noise. We observed that BMC preserves neighborhoods when it maps data onto an embedding better than the maps of the other DR methods of interest. Topological DR is a process of transforming high-dimensional data onto a low-dimensional embedding such that the data is homeomorphic \footnote{A homeomorphism is defined as an equivalence relation and a one-to-one correspondence between points in two topological spaces $\tau_1$ and $\tau_2$ that is continuous in both directions $\tau_1\rightarrow\tau_2$ and $\tau_2\rightarrow\tau_1$.} to the embedding. Since a homeomorphism can be explained as a continuous map of neighborhoods between data and embedding, neighboring points in the dataset should also be neighboring points in the embedding.  Thus, the stronger the homeomorphism the lesser the neighborhood preserving error that we have seen from BMC in Fig.~\ref{fig:face_image_err}.

In the examples of this paper, we set wide gaps between lower and upper bounds. Given these bounds, BMC revealed a low-rank recovery distance matrix which showed superior embedding than that of the other DR methods of interest. However, due to wide bounds, BMC might reveal a low-rank distance matrix that is not close enough to the real distance matrix and encounter a considerable error. This is clearly observed by both the clustering error ($\mathcal{E}_c$) in Fig.~\ref{fig:mnist} and the neighborhood preserving error ($\mathcal{E}_{\delta'}$) in Fig.~\ref{fig:face_image_err}. Although BMC is made to recover a low-rank recovery within bounds, we will be able to recover the distance matrix of interest by setting these bounds sufficiently narrow. For that, we plan to use narrow bounds by taking the curvature of the manifold into consideration. Specifically, Menger curvature routing \cite{pajot2003analytic}, is used to estimate the point-wise curvatures of the observed points on the manifold, then these curvatures are interpolated over the missing portion (like holes), if any, to estimate the rest of the point-wise curvatures. Based on the curvature, we will compute close lower and upper bounds for the missing part of the distance matrix and then recover it using BMC.

BMC reveals a PSD low-rank representation of a partially observed distance matrix whereas the recovered matrix satisfies user input bounds of the distances. However, it would be interesting to test if we could have recovered a low-rank representation $L$ of the squared distance matrix $D^2$ along with a sparse representation $S$ such that $D^2=L+S$. This modification enables us to extract not only the low-rank potion of the distance matrix but also the sparse portion that will help us to study anomalies in the dataset. Detecting anomalies is helpful in learning bank frauds, structural defects, medical problems, cyber attacks, etc. In this new MC scheme, we could utilize the same truncated nuclear norm along with the PSD constraint, as in the BMC scheme, to ensure the low-rankness of $L$, and we could employ $l_0$ norm of $S$ to extract a sparse matrix $S$ from $D^2$ as presented in \cite{wright2009robust}. However, $l_0$ norm of $S$ is non-convex and discontinuous that we could replace by the convex relaxation $l_1$ norm of $S$, \cite{wright2009robust}. This novel MC scheme is extended from BMC under two new modifications to the MC routing given in \eqref{eqn:optRouting1}: 
\begin{enumerate}
\item The objective function is \\$\underset{L, K, S}{\text{minimize}} \ \ \ \|L\|_* - \mathrm{tr}(U_rKV^T_r)+\beta \|S\|_1$;
\item The second constraint is $D^l\le K+S \le D^u$;
\end{enumerate}
where the parameter $\beta$, $0\le\beta\le1$, controls the trade-off between $L$ and $S$. 

We believe that this is the first time a PSD low-rank truncated MC technique that inputs the partially observed matrix as bounds has been developed.  Thus, we will be able to implement this technique for many applications in networking, financial market analysis, and collective motion. Specifically, we will utilize BMC technique to capture and characterize network topology of social networks such as Facebook, Arxiv, and Enron e-mail. We will also utilize BMC scheme to map network topology of 2-D and 3-D sensor networks from partial virtual coordinates to graph geodesics \cite{Jayasumana2018network}. In addition, we are interested in DR of collective motion such as schools of fish and flocks of birds. Our recent work, \cite{Gajamannage2015b, gajamannage2017detect}, concludes that the transitions in collective motion can be represented as switchings of the underlying manifolds which we will study further in an efficient manner by utilizing BMC technique. 

\ifCLASSOPTIONcompsoc
\section*{Acknowledgments}
 The authors would like to thank NSF XSEDE for allocating computational time  from the resource Jetstream under the request number DMS180007.
\else
 
  \section*{Acknowledgment}
\fi

\ifCLASSOPTIONcaptionsoff
  \newpage
\fi

\bibliographystyle{IEEEtran}


\end{document}